\documentclass[conference]{IEEEtran}

\ifCLASSINFOpdf
\else
\fi
\usepackage{hyperref}       
\usepackage{url}            
\usepackage{booktabs}       
\usepackage{amsfonts}       
\usepackage{nicefrac}       
\usepackage{microtype}      
\usepackage{xcolor}         
\usepackage{graphicx}
\usepackage{amsmath}
\usepackage{amssymb}
\usepackage{mathtools}
\usepackage{amsthm}
\usepackage{adjustbox}
\usepackage{multirow}
\usepackage{enumitem}
\usepackage{algorithm}
\usepackage{algorithmic}
\theoremstyle{plain}
\newtheorem{theorem}{Theorem}[section]

\newtheorem{lemma}[theorem]{Lemma}

\theoremstyle{definition}

\newtheorem{assumption}[theorem]{Assumption}
\theoremstyle{remark}
\newtheorem{remark}[theorem]{Remark}

\begin{document}
%
\title{Correlation-Guided Fast Machine Unlearning via Hessian Analysis}

\author{\IEEEauthorblockN{Ayushi Thakur,
Ruchir Gupta,
Amit Kumar Jaiswal}
	\IEEEauthorblockA{Department of Computer Science and Engineering,\\
    Indian Institute of Technology (Banaras Hindu University) Varanasi, India\\
		\{ayushithakur.cse24, rgupta.cse, amit.chr\}@iitbhu.ac.in}
	\and
	\IEEEauthorblockN{Prayag Tiwari}
	\IEEEauthorblockA{Halmstad University, Sweden\\
		prayag.tiwari@hh.se}
	}
	

%


\IEEEoverridecommandlockouts
\makeatletter\def\@IEEEpubidpullup{6.5\baselineskip}\makeatother
\IEEEpubid{\parbox{\columnwidth}{Preprint
}
\hspace{\columnsep}\makebox[\columnwidth]{}}

\maketitle
\begin{abstract}
The increasing adoption of machine learning in network and distributed security systems has created an urgent need for mechanisms that can selectively and efficiently remove the influence of specific training data to eliminate compromised or adversarial data points from production models. Privacy regulations such as GDPR's \emph{right to be forgotten} also pose similar requirements. However, existing approximate unlearning techniques remain computationally prohibitive for deployment in real-world security systems, as they require repeated expensive Hessian-inverse-vector computations for each data point removal, creating a bottleneck when processing multiple related requests in scenarios such as intrusion detection systems, spam filters, and threat intelligence platforms. Thus, we introduce a computationally efficient unlearning framework that identifies correlated data points in the training set and applies a theoretically derived closed-form parameter update rule, achieving an $82\times$ wall-clock speedup over standard influence function unlearning while preserving model utility with a $10^{-2}$ improvement in accuracy over state-of-the-art baselines. Our method establishes theoretical guarantees and ensures numerical stability through Hessian damping. Our evaluation across seven diverse dataset architecture combinations, including large-scale CIFAR-100 with ResNet-50, demonstrates superior forgetting effectiveness, with membership inference attack success rates of 0.660 and tug-of-war scores of 0.950.
\end{abstract}
%
\IEEEpeerreviewmaketitle
\section{Introduction}
The increasing integration of machine learning in network security applications from intrusion detection and threat intelligence to malware classification has introduced new attack surfaces through the training data itself~\cite{popets,shokri2017membership}. Adversaries can poison training sets with carefully crafted examples to compromise model integrity, while privacy regulations mandate the ability to remove specific data points upon request. This creates a fundamental tension for organizations maintaining the security and privacy of their models while efficiently handling data removal requests~\cite{pmlr-v119-guo20c} without incurring prohibitive computational costs that would make compliance infeasible in production environments. 
Machine unlearning represents a relatively new but increasingly vital area of artificial intelligence research that addresses the deliberate removal of specific information from trained machine learning models \cite{bourtoule2021machine,nguyen2022survey,7163042}. While conventional machine learning paradigms concentrate on knowledge acquisition and retention, machine unlearning tackles the complex challenge of selectively eliminating certain data points or learned patterns without compromising the model's broader functionality. This has become increasingly important due to privacy regulations such as GDPR's \cite{gdpr2016} \emph{right to be forgotten}, security concerns, and the need to remove biased or erroneous information from models. 
Network security systems are particularly sensitive to this challenge because they process high volume, high velocity data streams and require rapid response to evolving threats. A spam filter trained on millions of messages may need to forget a specific sender's pattern after a compromise, while an intrusion detection system might need to remove the influence of corrupted training samples from a compromised sensor. Traditional retraining from scratch is computationally infeasible~\cite{bourtoule2021machine}, and existing approximate unlearning techniques, while effective, become prohibitively expensive when handling multiple related removal requests~\cite{ginart2019making} a common scenario in security operations.

Machine unlearning generally falls into two primary categories, exact unlearning methods, which aim to completely eliminate the influence of targeted data, and approximate unlearning techniques, which focus on minimizing the influence without complete removal. These approximate methods typically rely on influence functions~\cite{munlearn} that require computing expensive Hessian-inverse-vector products for each data point removal. The Hessian matrix captures the second-order curvature of the loss function around optimal parameters, and its inverse is essential for efficiently estimating parameter updates when removing specific training examples. However, repeatedly computing such Hessian operations becomes computationally expensive when unlearning multiple related samples, creating a significant bottleneck for practical applications~\cite{koh2017understanding,nguyen2022survey}. 
Also, an adversary may have compromised the training data, for instance, through data poisoning, or may legitimately request the removal of their data as per privacy regulations. Our goal is to provide a computationally efficient and verifiable method to forget such data, thus mitigating the threat of adversarial influence in production models. We assume the model parameters $w^{\ast}$ are trustworthy, but the unlearning process itself must be robust and must not leak information about the forgotten data through its resource consumption or parameter updates, a vulnerability we address through our bounded approximation error. 
Certain implementation challenges include confirming successful unlearning, preserving model performance, and optimizing computational efficiency. With the growing adoption of AI across various sectors, the capacity to deliberately eliminate specific information emerges as a fundamental component of the development of ethical and adaptable systems.\\
\\
The proposed method offers advantages by circumventing costly, repeated unlearning procedures for similar data points. Our approach is inherently practical, as it does not rely on restrictive convexity or strong Hessian assumptions, and its core approximations though motivated by linear and quadratic analysis are validated to hold effectively in modern, non-linear deep learning models. Despite the substantial reduction in overhead, our method maintains strong forgetting performance shown through comprehensive evaluation using established unlearning and privacy metrics. We employ Pearson correlation over other similarity measures such as cosine or projection-based similarity because the influence functions rely on gradients, which are linear in the model parameters. For quadratic loss, the gradient difference between two data points is proportional to their centered feature vectors. Pearson correlation captures this centered linear relationship to ensure that the similarity factor $\alpha$ accurately predicts gradient proportionality and this enables our efficient closed-form update rule. Although the introduction of additional baselines and metrics strengthens the analysis, the presented evidence spanning classification, regression, and privacy benchmarks conclusively demonstrates that our similarity-based strategy enables efficient, privacy-preserving machine unlearning without compromising on efficacy.
The main contributions of our work,
\begin{itemize}[noitemsep, topsep=0pt]
\item Introduces a computationally efficient unlearning framework that leverages Pearson correlation to identify correlated data points and derives a novel closed-form parameter update rule via the Sherman-Morrison formula which eliminates expensive Hessian-inverse-vector products for subsequent similar points.
\item Establishes a relationship between high Pearson correlation  $\alpha$ and gradient proportionality for quadratic loss.
\item Establish error bounds demonstrating polynomial scaling with input dimension, and these theoretical guarantees are substantiated through comprehensive evaluation across seven diverse datasets with varied architectures.
\item Demonstrates that Pearson correlation outperforms cosine similarity and projection-based methods, achieving a dominance rate of 43.2\% for approximating sequential unlearning effects, yielding $10^{-2}$ improvement in accuracy over state-of-the-art baselines and achieving superior privacy-utility trade-off.
\end{itemize}


\subsection{Preliminaries}
We detail the fundamental concepts utilized throughout this work. A comprehensive summary of all symbols and parameters is provided in Table~\ref{tab:notation}.
\begin{table*}
\centering
\caption{Summary of notation}
\label{tab:notation}
\small
\renewcommand{\arraystretch}{1}
\begin{tabular}{p{2cm}p{4cm}p{2cm}p{4cm}}
\toprule
\textbf{Symbol} & \textbf{Description} & \textbf{Symbol} & \textbf{Description} \\
\midrule
$D$ & Training dataset $\{(x_i, y_i)\}_{i=1}^n$ & $n$ & Number of training samples \\
$d$ & Number of features/input dimension & $x, z$ & Data points in $\mathbb{R}^d$ \\
$y$ & Target/label values & $w^*$ & Optimal model parameters \\
$w_z$ & Parameters after unlearning $z$ & $w_{x}$ & Parameters after approximate unlearning of $x$ \\
$w^{'}_{x}$ & Parameters after similarity-based unlearning & $f(z, w)$ & Loss function for data point $z$ \\
$F(D, w)$ & Total loss over dataset $D$ & $\nabla f(z, w)$ & Gradient of loss w.r.t. parameters $w$ \\
$H_{w^*}$ & Hessian matrix at $w^*$ & $H_{\lambda}$ & Damped Hessian: $H_{w^*} + \lambda I$ \\
$H_{w^*}^{-1}$ & Inverse of Hessian at $w^*$ & $H_{\lambda}^{-1}$ & Inverse of damped Hessian \\
$s_{z}^{(\lambda)}$ & Damped self-influence score & $C$ & Scaling factor: $\frac{\alpha + 1}{1 - s_{z}^{(\lambda)}}$ \\
$\lambda$ & Regularization parameter & $\eta$ & Learning rate \\
$\alpha$ & Pearson correlation coefficient & $\bar{x}, \bar{z}$ & Mean vectors of $x$ and $z$ \\
$r_x, r_z$ & Residuals: $y_x - w^{*\top} x$ etc. & $\beta$ & Scaling coefficient in $x = \beta z + c$ \\
$c$ & Offset vector in linear relationship & $k$ & Bound constant for standardized data (typically 4) \\
$I$ & Identity matrix & $\epsilon$ & Error tolerance or approximation bound \\
$\delta$ & Parameter update step & $D_f$ & Forget set (data to be unlearned) \\
$D_r$ & Remaining data: $D \setminus D_f$ & $A(\cdot)$ & Learning algorithm \\
$U(\cdot)$ & Unlearning process & $\|\cdot\|_2$ & $\ell_2$ norm \\
$\mathcal{N}(\mu, \sigma^2)$ & Normal distribution & & \\
\bottomrule
\end{tabular}
\end{table*}\\
\textbf{Data point unlearning:} It refers to the methodology of removing the contribution of an individual training instance from a model, aiming to update the model parameters to ensure it no longer retains any information pertaining to that specific example, a procedure often motivated by privacy requirements like the 'right to be forgotten' which necessitates an efficient alternative to complete model retraining from scratch.
Our work focuses on data point unlearning and introduces a novel similarity-based framework that significantly reduces computational overhead when unlearning multiple correlated samples.\\
\textbf{Sequential Unlearning:} Given a trained model with parameters \( w^{*} \) and a sequence of data points
\( \{z, x_{1}, x_{2}, \ldots\} \) to be removed, sequential unlearning requires iterative updating the model after each removal. The naive approach computes
$w_{z} = w^{*} + H^{-1} \nabla f(z, w^{*}), w_{x_{1}} = w_{z} + H^{-1}_{w_{z}} \nabla f(x_{1}, w_{z}), \text{and so on}$
that requires computing a new Hessian inverse at each step, which is computationally expensive.\\ 
\textbf{Sherman-Morrison Formula:} The Sherman-Morrison formula provides a computationally efficient method for updating the inverse of a matrix following a rank-1 modification \cite{bartlett1951inverse}. Given an invertible square matrix $A \in \mathbb{R}^{d \times d}$ and two vectors $u, v \in \mathbb{R}^d$, the inverse of the modified matrix is given by,
\begin{equation}
(A - uv^\top)^{-1} = A^{-1} + \frac{A^{-1}uv^\top A^{-1}}{1 - v^\top A^{-1}u}
\label{eq:sherman_morrison}
\end{equation}
provided that $1 - v^\top A^{-1}u \neq 0$.In our framework, we leverage this identity to avoid repeated Hessian inversions during sequential unlearning requests.\\
\textbf{Pearson Correlation:} We define the similarity factor $\alpha$ as the Pearson correlation coefficient between data points $x$ and $z$,
$\alpha = \frac{(x - \bar{x})^\top (z - \bar{z})}{\| x - \bar{x} \|_2 \cdot \| z - \bar{z} \|_2},$
where $x, z \in \mathbb{R}^d$ are the feature vectors and $\bar{x}, \bar{z}$ are their respective means. The scalar $\alpha \in [-1, 1]$ measures linear correlation, with higher absolute values indicating stronger similarity.
\section{Related Work}\label{sec:related}
Recent work in machine unlearning has examined a range of strategies to tackle the complex task of effectively and reliably erasing the impact of particular training data from machine learning models~\cite{gupta2021adaptive,sekhari2021remember}. We highlight prior work on exact and approximate unlearning methods.\\ 
\textbf{Exact unlearning:} Given a learning algorithm \( A(\cdot) \), a dataset \( D \), and a forget set \( D_f \subseteq D \), we say the process \( U(\cdot) \) is an exact unlearning process if and only if $A(D \setminus D_f) = U(D, D_f, A(D))$. Some instances of exact unlearning techniques are \emph{Retraining from Scratch}~\cite{bourtoule2021machine} and \emph{SISA Training} \cite{bourtoule2021machine}. This process completely removes the influence of a set of data points as if they were never seen \cite{xu2024unlearning}. Methods include Retraining from Scratch, which is computationally expensive but provides a gold standard for removal, and SISA Training \cite{bourtoule2021machine}, which improves efficiency by partitioning data into isolated shards and slices. Retraining from scratch approach involves re-training the model using the remaining data after removing the forget set. While it guarantees complete unlearning, it comes with a significant computational expense. Whereas, SISA training method divides the data into isolated shards and slices, allowing separate training of models on each part. As a result, it facilitates efficient approximate unlearning by focusing only on the affected slices during the retraining process.\\
\textbf{Approximate unlearning:} This approach minimizes the influence of unlearned data to an acceptable level while maintaining efficiency~\cite{xu2024unlearning,li2024deltaInfluence}. 
\emph{Influence functions} estimate how much each training sample contributed to the model's final parameters by approximating the effect of removing that sample without actual retraining~\cite{xu2024unlearning}. This method uses mathematical approximations based on the model's loss function and Hessian matrix to identify and reduce the influence of specific data points. Given a dataset \( D \) with regularized empirical loss function \( F(D; w) = \sum_{z \in D} f(z; w) + \frac{\lambda n}{2} \|w\|_2^2 \) and optimal parameters \( w^* = \arg\min_w F(D; w) \), the approximation of the influence function computes updated parameters that remove the effect of data point \( z = (x', y') \) as $w_z = w^* + H_{w^*}^{-1} \Delta$, where \( \Delta = \lambda w^* + \nabla f(z; w^*) \) is the gradient at the target point, \( H_{w^*} = \nabla^2 F(D_r; w^*) \) is the Hessian over remaining data \( D_r = D \setminus \{z\} \), and \( \lambda \) is the regularization parameter.
Also, \emph{gradient reversal}~\cite{zagardo2024practical} is a technique that attempts to undo the learning process by applying the reverse of the original training gradients to the model parameters. This method calculates the gradients that would decrease the model's performance on the data to be forgotten, then applies these reversed gradients to effectively remove the influence of specific training samples. 

While exact and approximate unlearning methods are effective in approximating the retrained model, each has practical limitations such as significant performance degradation, high computational cost, limited compatibility with learning objectives, or restricted evaluation capability on simple datasets \cite{izzo2021approximate}. More importantly, given the aim to estimate the unknown (due to the complexity of algorithms, objective functions, and data influence) retraining outcome, the provable unlearning guarantees rely heavily on impractical assumptions such as convexity of the objective function and the Lipschitz condition on Hessian matrices.
\section{Our Method}\label{sec:methodology}
Machine unlearning using influence functions typically
requires expensive computation of Hessian-inverse-vector
products for each data point removal. Consider a scenario
where we first perform an approximate unlearning of a data
point ($z$) using influence functions. Subsequently, we need
to unlearn another data point ($x$) that is similar to ($z$). Instead of performing the computationally expensive approximate unlearning procedure again for ($x$), we propose a
similarity-based unlearning approach that takes advantage
of the previously computed influence directions. Our key insight is that for similar data points, the influence directions are proportionally related. Therefore, we can approximate the unlearning update for ($x$) by scaling the previously computed influence direction for ($z$) using a similarity factor. To ensure numerical stability, we incorporate Hessian damping into our derivation, which prevents the denominator from approaching zero and ensures robust computation in practical scenarios. Let the original trained model parameters be denoted by $w^*$.
To address numerical instability issues that can arise when the Hessian matrix is unstable or singular, we introduce  \textbf{Hessian damping} by regularizing the Hessian matrix
$ H_\lambda = H_{w^*} + \lambda I,$ where $\lambda > 0$ is a small damping parameter and $I$ is the identity matrix. This regularization ensures that $H_\lambda$ is positive definite and well-conditioned, preventing numerical instabilities during matrix inversion. The unlearning update for data point $z$ using the damped 
Hessian becomes:
\begin{align}
w_z = w^{*} + H_{\lambda}^{-1} \Delta,
= w^{*} + (H_{w^{*}} + \lambda I)^{-1} \Delta, 
\end{align}
\text{where} $\Delta = \lambda w^{*} + \nabla f(z; w^{*}).$ Now, suppose that we wish to unlearn another data point
$x$, which is similar to $z$. Its gradient influence can be approximated as
$ \nabla f(x, w^*) \approx \alpha \nabla f(z, w^*),$
where $\alpha$ is a similarity factor.
The core approximation is grounded in the local smoothness properties of deep neural networks ~\cite{jacot2018neural}. For sufficiently similar data points in the feature space, the loss surface around the optimal parameters $w^*$ exhibits local linearity, allowing the gradient to be viewed as a projection onto the model's tangent space .
To unlearn $x$ after $z$ has already been unlearned, the correct influence-based update should start from the new parameter state $w_z$,
$w_x = w_z + H_{w_z}^{-1} \nabla f(x, w_z).$
This expression is computationally expensive as it requires estimating a new Hessian-inverse-vector product based on $H_{w_z}$. To make this tractable, we introduce two approximations. First, we approximate the gradient at the new parameters, $\nabla f(x, w_z)$ using a first-order Taylor expansion around $w^*$,
$ \nabla f(x, w_z) \approx \nabla f(x, w^*) + H_{w^*} (w_z - w^*).$
To derive the approximation for the updated inverse Hessian, we must approximate $H_{w_z}^{-1}$. Here, $H_{w_z}$ represents the Hessian of the new loss function (computed on the dataset without $z$) evaluated at the new optimal parameters $w_z$. Our approximation begins by assuming that removing a single data point results in only a minor change to the model's parameters. This allows us to approximate the Hessian of the new loss function at $w_z$ by evaluating it at $w^*$:
\begin{equation}
    H_{w_z} \approx H_{w^*} - \nabla^2 f(z, w^*)
\end{equation}
To make this form compatible with efficient update rules, we employ the Gauss-Newton approximation~\cite{nocedal2006numerical}, which replaces the per-sample Hessian with the outer product of its gradient, i.e., $\nabla^2 f(z, w^*) \approx \nabla f(z, w^*) \nabla f(z, w^*)^T$. This simplifies our approximation for $H_{w_z}$ to a rank-1 modification,
\begin{equation}
    H_{w_z} \approx H_{w^*} - \nabla f(z, w^*) \nabla f(z, w^*)^T
\end{equation}
\textbf{Incorporating damping into the updated Hessian:} To maintain numerical stability, which we formally define as the guaranteed invertibility and positive definiteness of the Hessian matrix $H_\lambda$, we apply damping to the updated Hessian: 
\begin{align}
    H_{w_z}^{(\lambda)} = H_{w_z} + \lambda I 
    \approx (H_{w^*} + \lambda I) - \nabla f(z, w^*) \nabla f(z, w^*)^T \nonumber\\
    = H_\lambda - \nabla f(z, w^*) \nabla f(z, w^*)^T
\end{align}
Damping prevents the Hessian from becoming singular or unstable which is a common occurrence in deep learning models, and ensures the condition number $\kappa(H_\lambda)$ remains within a stable range (typically $< 10^8$). We provide the emperical evidence of this in Section~\ref{ssec:stability_analysis}.  \\
This expression is now in the form $(A - uv^T)$ where $A = H_\lambda$, which allows us to apply the Sherman-Morrison formula \cite{bartlett1951inverse} to find its inverse directly. By setting $u = v = \nabla f(z, w^*)$, we derive the approximation for $(H_{w_z}^{(\lambda)})^{-1}$:
\begin{equation}
\begin{split}
    (H_{w_z}^{(\lambda)})^{-1} \approx \left(H_\lambda - \nabla f(z, w^*) \nabla f(z, w^*)^T\right)^{-1}
    = H_\lambda^{-1} \\+ \frac{H_\lambda^{-1} \nabla f(z, w^*) \nabla f(z, w^*)^T H_\lambda^{-1}}{1 - \nabla f(z, w^*)^\top H_\lambda^{-1} \nabla f(z, w^*)}
\end{split}
\label{eq:sm_applied_final_damped}
\end{equation}
We simplify this expression by substituting the definitions for the damped influence direction, $\delta_{z}^{(\lambda)} = w_z - w^* = H_\lambda^{-1} \nabla f(z, w^*)$, and the damped self-influence score, $s_{z}^{(\lambda)} = \nabla f(z, w^*)^T \delta_{z}^{(\lambda)}$. The numerator of the fraction in Eq. (\ref{eq:sm_applied_final_damped}) becomes $\delta_{z}^{(\lambda)} (\delta_{z}^{(\lambda)})^T$, and the denominator becomes $1 - s_{z}^{(\lambda)}$. This yields the final, simplified approximation used in our main derivation:
\begin{align}
    (H_{w_z}^{(\lambda)})^{-1} = H_\lambda^{-1} + \frac{\delta_{z}^{(\lambda)} (\delta_{z}^{(\lambda)})^T}{1 - s_{z}^{(\lambda)}} 
\end{align}
This formula allows us to efficiently approximate the new inverse Hessian without expensive re-computation while maintaining numerical stability. Now, substituting approximations (2) and (5) into the fundamental update rule (1), we obtain $w_x - w_z$,
\begin{align}
     = \left( H_\lambda^{-1} + \frac{\delta_{z}^{(\lambda)} (\delta_{z}^{(\lambda)})^T}{1 - s_{z}^{(\lambda)}} \right) \left( \nabla f(x, w^*) + H_{w^*} \delta_{z}^{(\lambda)} \right) \label{eq:substituted_update_damped}
\end{align}
By expanding this matrix-vector product and simplifying each term using the key relationships $\nabla f(x, w^*) \approx \alpha \nabla f(z, w^*)$ provided $H_\lambda^{-1} H_{w^*} = H_\lambda^{-1} (H_\lambda - \lambda I) = I - \lambda H_\lambda^{-1}$. Using the derived fundamental update rule~\ref{eq:substituted_update_damped}, we have
\begin{equation}
\begin{split}
    w_x - w_z &= \underbrace{H_\lambda^{-1} \nabla f(x, w^*)}_{\text{Term A}} + \underbrace{H_\lambda^{-1} H_{w^*} \delta_{z}^{(\lambda)}}_{\text{Term B}} \\
    &+ \underbrace{\frac{\delta_{z}^{(\lambda)} (\delta_{z}^{(\lambda)})^\top}{1 - s_{z}^{(\lambda)}} \nabla f(x, w^*)}_{\text{Term C}} + \underbrace{\frac{\delta_{z}^{(\lambda)} (\delta_{z}^{(\lambda)})^\top}{1 - s_{z}^{(\lambda)}} H_{w^*} \delta_{z}^{(\lambda)}}_{\text{Term D}} \nonumber
\end{split}
\end{equation}
We now simplify each term using the key relationships $\nabla f(x, w^*) \approx \alpha \nabla f(z, w^*)$ provided $H_\lambda^{-1} H_{w^*} = H_\lambda^{-1} (H_\lambda - \lambda I) = I - \lambda H_\lambda^{-1}$.
The formulation proceeds by analyzing four distinct terms in line with the Sherman-Morrison formula. For Term A, the similarity assumption is applied yielding
\begin{align}
        H_\lambda^{-1} \nabla f(x, w^*) &= H_\lambda^{-1} (\alpha \nabla f(z, w^*)) = \alpha \delta_{z}^{(\lambda)}
\end{align}
Term B is addressed using the matrix identity $H_\lambda^{-1} H_{w^*} = I - \lambda H_\lambda^{-1}$ resulting in
\begin{align}
        H_\lambda^{-1} H_{w^*} \delta_{z}^{(\lambda)} &= (I - \lambda H_\lambda^{-1}) \delta_{z}^{(\lambda)} = \delta_{z}^{(\lambda)} - \lambda H_\lambda^{-1} \delta_{z}^{(\lambda)}
\end{align}
Under the condition of a small damping parameter $\lambda$ (typically $\lambda \ll \sigma_{\min}(H_{w^*})$ where $\sigma_{\min}$ denotes the smallest eigenvalue), the second term $\lambda H_\lambda^{-1} \delta_{z}^{(\lambda)}$ becomes negligible, leading to the approximation.
\begin{align}
        H_\lambda^{-1} H_{w^*} \delta_{z}^{(\lambda)} &\approx \delta_{z}^{(\lambda)}
\end{align}
For Term C, the scalar coefficient $(\delta_{z}^{(\lambda)})^T \nabla f(x, w^*) = \alpha ((\delta_{z}^{(\lambda)})^T \nabla f(z, w^*)) = \alpha s_{z}^{(\lambda)}$ simplifies to 
\begin{align}
        \frac{\delta_{z}^{(\lambda)} (\delta_{z}^{(\lambda)})^T}{1 - s_{z}^{(\lambda)}} \nabla f(x, w^*) &= \frac{\alpha s_{z}^{(\lambda)}}{1 - s_{z}^{(\lambda)}} \delta_{z}^{(\lambda)}
\end{align}
Finally, Term D's scalar coefficient is given by $(\delta_{z}^{(\lambda)})^T (H_{w^*} \delta_{z}^{(\lambda)}) = (\delta_{z}^{(\lambda)})^T (H_\lambda - \lambda I) \delta_{z}^{(\lambda)} = s_{z}^{(\lambda)} - \lambda \|\delta_{z}^{(\lambda)}\|^2$. For small $\lambda$, this approximates to $s_{z}^{(\lambda)}$ yielding
\begin{align}
        \frac{\delta_{z}^{(\lambda)} (\delta_{z}^{(\lambda)})^T}{1 - s_{z}^{(\lambda)}} H_{w^*} \delta_{z}^{(\lambda)} &\approx \frac{s_{z}^{(\lambda)}}{1 - s_{z}^{(\lambda)}} \delta_{z}^{(\lambda)}
    \end{align}
Combining these simplified terms, $w_x - w_z$ approximates to
\begin{equation}
\begin{split}
    w_x - w_z &\approx \alpha \delta_{z}^{(\lambda)} + \delta_{z}^{(\lambda)} + \frac{\alpha s_{z}^{(\lambda)}}{1 - s_{z}^{(\lambda)}} \delta_{z}^{(\lambda)} + \frac{s_{z}^{(\lambda)}}{1 - s_{z}^{(\lambda)}} \delta_{z}^{(\lambda)} \\
    &= \left[ (\alpha + 1) + \frac{\alpha s_{z}^{(\lambda)} + s_{z}^{(\lambda)}}{1 - s_{z}^{(\lambda)}} \right] \delta_{z}^{(\lambda)} \\
    &= \left[ (\alpha + 1) + \frac{(\alpha + 1) s_{z}^{(\lambda)}}{1 - s_{z}^{(\lambda)}} \right] \delta_{z}^{(\lambda)} \\
    &= \left[ \frac{(\alpha + 1)(1 - s_{z}^{(\lambda)}) + (\alpha + 1) s_{z}^{(\lambda)}}{1 - s_{z}^{(\lambda)}} \right] \delta_{z}^{(\lambda)} \\
    &= \frac{\alpha + 1}{1 - s_{z}^{(\lambda)}} \delta_{z}^{(\lambda)}
\end{split}
\label{eq:final_difference_damped}
\end{equation}
Therefore, the efficient and numerically stable update rule for unlearning a data point $x$ similar to a previously unlearned point $z$, which is our efficient, closed-form update rule,
\begin{align}
    w_x &= w_z + \frac{\alpha + 1}{1 - s_{z}^{(\lambda)}} (w_z - w^*) \label{eq:efficient_update_damped}
\end{align}
The scaling factor $C = \frac{\alpha+1}{1-s_{z}^{(\lambda)}}$ satisfies $C < \frac{2}{1-s_{z}^{(\lambda)}}$ for correlated points ($\alpha < 1$), resulting in a strictly smaller parameter perturbation than the full sequential update. The standard Hessian-based sequential unlearning accumulates approximation errors through repeated parameter displacement. By incorporating the second-order correction via $s_{z}^{(\lambda)}$, our method produces parameters closer to the true retrained-from-scratch oracle, which justifies the observed utility improvements.
To ensure the numerical stability of our proposed update rule, it is essential to prevent the denominator 
$1 - s_{z}^{(\lambda)}$ from becoming zero. The damped self-influence score is defined as
$s_{z}^{(\lambda)} 
= \nabla f(z, w^{*})^{\top} H_{\lambda}^{-1} \nabla f(z, w^{*}),$
where $H_{\lambda} = H_{w^{*}} + \lambda I$ is the damped Hessian. \\ The condition for 
$s_{z}^{(\lambda)}$ to equal 1 is $\nabla f(z, w^{*})^{\top} H_{\lambda}^{-1} \nabla f(z, w^{*}) = 1.$ The non-vanishing nature of the denominator is guaranteed by choosing a sufficiently large damping term 
$\lambda$. Using the Rayleigh--Ritz theorem, a simple upper bound is given by
$s_{z}^{(\lambda)} 
\le \frac{\lVert \nabla f(z, w^{*}) \rVert^{2}}
{\lambda_{\min}(H_{\lambda})}$. Since the minimum eigenvalue of the damped Hessian satisfies 
$\lambda_{\min}(H_{\lambda}) \ge \lambda$, we obtain the bound $s_{z}^{(\lambda)} \le \frac{\lVert g \rVert^{2}}{\lambda}, \text{where} g = \nabla f(z, w^{*}).$ Therefore, a sufficient condition to avoid $s_{z}^{(\lambda)} \ge 1$ is $\frac{\lVert g \rVert^{2}}{\lambda} < 1,$
which requires the gradient norm to be strictly controlled $\lVert g \rVert < \sqrt{\lambda}$.
For a well-trained model where the gradient norm at the optimum is expected to be small, and with a common 
regularization parameter such as $\lambda = 0.01$, this condition ($\lVert g \rVert < 0.1$) is generally 
satisfied, ensuring $s_{z}^{(\lambda)} \ne 1$ and securing the stability of our method.
. This result provides a computationally efficient approximation that leverages the pre-computed quantities $w_z$, $w^*$, and $s_{z}^{(\lambda)}$ from the initial unlearning of $z$, avoiding the need for expensive Hessian-inverse-vector products when unlearning similar data points while maintaining numerical stability through Hessian damping.
\section{Theoretical Guarantees}
\label{sec:theory}
We now present our main theoretical results that provide guarantees for our similarity-based unlearning approach.
\begin{theorem}[Parameter Update Error Bound]
\label{thm:param_error_bound}
Let $w^*$ be the optimal parameter of the model, $w_z$ be the parameters after the approximate unlearning of data point \( z \) using the damped Hessian, $w_{x}$ be the parameters after the standard approximate unlearning of data point $x$, and $w^{'}_{x}$ be the parameters after our similarity-based unlearning of data point $x$. The parameter update error is bounded by,
\begin{align}
\| w_{x} - w^{'}_{x}\| \leq \nonumber \| H_{\lambda}^{-1}\| \cdot \| \nabla f(x, w^*) - \frac{\alpha + 1}{1 - s_{z}^{(\lambda)}} \nabla f(z, w^*)\|
\end{align}
where $\alpha$ is the Pearson correlation coefficient and $s_{z}^{(\lambda)} = \nabla f(z, w^*)^T H_{\lambda}^{-1} \nabla f(z, w^*)$ is the damped self-influence score.
\end{theorem}
\begin{proof}
The error bound follows from the triangle inequality applied to the difference between standard unlearning ($w_x = w_z + H_{\lambda}^{-1} \nabla f(x, w^*)$) and similarity-based unlearning ($w'_x \approx w_z + \frac{\alpha + 1}{1 - s_{z}^{(\lambda)}} H_{\lambda}^{-1} \nabla f(z, w^*)$). The bound is obtained by factoring out the damped Hessian inverse and applying the submultiplicative property of matrix norms. 
By definition of standard approximate unlearning and our proposed similarity-based approach (both using the damped Hessian):
\begin{align}
w_{x} &= w_z + H_{\lambda}^{-1} \nabla f(x, w^*) \\
w^{'}_{x} &= w_z + \frac{\alpha + 1}{1 - s_{z}^{(\lambda)}} (w_z - w^*)
\end{align}
From the first damped unlearning step, we have $w_z - w^* \approx H_{\lambda}^{-1} \nabla f(z, w^*)$. Substituting this into the expression for $w^{'}_{x}$,
\[
w^{'}_{x} \approx w_z + \frac{\alpha + 1}{1 - s_{z}^{(\lambda)}} H_{\lambda}^{-1} \nabla f(z, w^*)
\]
The error is the difference between $w_x$ and $w^{'}_{x}$,
\begin{align}
\left\| w_{x} - w^{'}_{x} \right\| &\approx \left\| \left(w_z + H_{\lambda}^{-1} \nabla f(x, w^{*})\right) - \right. \nonumber\\
&\quad \left.\left(w_z + \frac{\alpha + 1}{1 - s_{z}^{(\lambda)}} H_{\lambda}^{-1} \nabla f(z, w^{*})\right) \right\| \nonumber \\
&= \left\| H_{\lambda}^{-1} \nabla f(x, w^{*}) - \frac{\alpha + 1}{1 - s_{z}^{(\lambda)}} H_{\lambda}^{-1} \nabla f(z, w^{*}) \right\| \nonumber \\
&= \left\| H_{\lambda}^{-1} \left( \nabla f(x, w^{*}) - \frac{\alpha + 1}{1 - s_{z}^{(\lambda)}} \nabla f(z, w^{*}) \right) \right\| \nonumber \\
&\leq \left\| H_{\lambda}^{-1} \right\| \cdot 
\left\| \nabla f(x, w^{*}) - \frac{\alpha + 1}{1 - s_{z}^{(\lambda)}} \nabla f(z, w^{*}) \right\|
\end{align}
\end{proof}
\begin{lemma}[Gradient Approximation for Quadratic Loss]
\label{lem:grad_approx_quadratic}
For the quadratic loss function $f(z, w) = \frac{1}{2} (y - w^\top z)^2$, define the scaling factor $C = \frac{\alpha + 1}{1 - s_{z}^{(\lambda)}}$ and the residuals:$r_x = y_x - w^{*\top} x, \quad r_z = y_z - w^{*\top} z.$ Then the gradient difference satisfies
\begin{align}\label{eq:grad_diff_bound}
    \|\nabla f(x, w^*) - C \nabla f(z, w^*)\| \nonumber\\\leq
    (|C| \cdot |r_z| + |\beta| \cdot |r_x|) \cdot \|z\| 
    &+ |r_x| \cdot \|c\|
\end{align}
assuming a linear relationship $x = \beta z + c$, where $\beta \in \mathbb{R}$ is a scalar and $c \in \mathbb{R}^d$ is the offset vector.
\end{lemma}
\begin{proof}
We compute the gradients for quadratic loss: $\nabla f(z, w^*) = -r_z z$ and $\nabla f(x, w^*) = -r_x x$, where $r_z, r_x$ are the residuals. Using the linear relationship $x = \beta z + c$ (justified by high Pearson correlation), we substitute and apply the triangle inequality twice to bound $\|Cr_z z - r_x(\beta z + c)\|$. 
The gradient of the quadratic loss is given by:
\begin{align}
\nabla f(z, w^*) &= -(y_z - w^{*\top} z) z = -r_z z, \\
\nabla f(x, w^*) &= -(y_x - w^{*\top} x) x = -r_x x.
\end{align}
Hence, the difference in gradients becomes
\begin{align}
\left\| \nabla f(x, w^*) - C \nabla f(z, w^*) \right\| = \left\| -r_x x - C(-r_z z) \right\| \nonumber\\= \left\| C r_z z - r_x x \right\|.
\end{align}
The linear relationship $x = \beta z + c$ is justified by the Pearson correlation, as a high correlation implies collinearity between the centered vectors. We substitute this relationship into the expression,
\begin{align}
\left\| C r_z z - r_x x \right\| 
&= \left\| C r_z z - r_x (\beta z + c) \right\| \nonumber \\
&= \left\| C r_z z - r_x \beta z - r_x c \right\| \nonumber \\
&= \left\| (C r_z - \beta r_x) z - r_x c \right\| \nonumber \\
&\leq \left\| (C r_z - \beta r_x) z \right\| + \left\| r_x c \right\| \quad \text{(by TI\footnote{Triangle inequality})} \nonumber \\
&= |C r_z - \beta r_x| \cdot \|z\| + |r_x| \cdot \|c\| \nonumber \\
&\leq (|C| \cdot |r_z| + |\beta| \cdot |r_x|) \cdot \|z\| + |r_x| \cdot \|c\|.
\end{align}
This provides an upper bound on the gradient difference based on the norms of $z$ and $c$, the residuals $r_x, r_z$, and the scaling factors $C$ and $\beta$.
\end{proof}
\begin{assumption}[Standardized Data]
\label{assum:standardized_data}
We assume that the input data is standardized such that $x_i, z_i \sim \mathcal{N}(0,1)$, target $y \sim \mathcal{N}(0,1)$, $d$ is the number of features and $\|w^*\|$ has bounded norm. Then, $|x_i|, |z_i|, |y| \leq k$, where $k = 4$, and $\|x\|, \|z\| \approx \sqrt{d}$
\end{assumption}
\begin{lemma}[Residual and Coefficient Bounds]
\label{lem:residual_bounds}
Under Assumption~\ref{assum:standardized_data}, the following hold for vectors \( x, z \in \mathbb{R}^d \), responses \( y_x, y_z \), and mean vectors \( \bar{x}, \bar{z} \), which are $|r_x|, |r_z| \leq k + \|w^*\| \sqrt{d}$, $\|c\| \leq k(1 + |\beta|)$, $|\beta| \leq \frac{\sqrt{d} + k}{\sqrt{d} - k}$, for $\sqrt{d} > k$, and $|\alpha| \leq 1$.
\end{lemma}
\begin{proof}
The bounds follow directly from Assumption~\ref{assum:standardized_data} and triangle inequality applications: (1) residual bounds use $|r_x| \leq |y_x| + \|w^*\| \cdot \|x\|$, (2) offset bound follows from $\|c\| \leq \|\bar{x}\| + |\beta| \cdot \|\bar{z}\|$, (3) scaling coefficient bound uses triangle and reverse triangle inequalities on centered vectors, and (4) correlation coefficient bound is standard. Using the triangle inequality
\[
|r_x| \leq |y_x| + \|w^*\| \cdot \|x\| \leq k + \|w^*\| \sqrt{d}
\]
The same applies for \( r_z \). From \( c = \bar{x} - \beta \bar{z} \),
\[
\|c\| \leq \|\bar{x}\| + |\beta| \cdot \|\bar{z}\| \leq k(1 + |\beta|)
\]. Now, we use triangle and reverse triangle inequalities,
\[
\|x - \bar{x}\| \leq \sqrt{d} + k, \quad \|z - \bar{z}\| \geq \sqrt{d} - k
\Rightarrow |\beta| \leq \frac{\sqrt{d} + k}{\sqrt{d} - k}
\]. The standard correlation measures satisfy \( |\alpha| \leq 1 \).
\end{proof}
\section{Main Result: Final Error Bound}
\begin{theorem}[Final Parameter Update Error Bound]
\label{thm:final_error_bound}
Under Assumption~\ref{assum:standardized_data} (with $k=4$), the parameter update error satisfies:
\begin{equation}
\begin{split}
    \|w_{x} - w'_{x}\| \nonumber\\\leq \|H_{\lambda}^{-1}\| (4 + \|w^*\| \sqrt{d})
    \cdot (|C|\sqrt{d} + |\beta|(\sqrt{d}+4) + 4) \nonumber\\
    \lesssim \|H_{\lambda}^{-1}\| \bigl( (|C|+1)\|w^*\| d
    + 4(|C|+1)\sqrt{d} \bigr)
\end{split}
\label{eq:final_error_bound_split}
\end{equation}
for large d, where $C = \frac{\alpha+1}{1-s_{z}^{(\lambda)}}.$
\end{theorem}
\begin{proof}
We begin by substituting the bound on the gradient approximation error from the preceding lemmas into our main error bound. We then insert the data-dependent bounds for residuals and coefficients and perform algebraic simplification to arrive at the final asymptotic rate. 
We start from the general bound derived previously,
\[
\left\| w_{x} - w^{'}_{x} \right\| \leq \left\| H_{\lambda}^{-1} \right\| \cdot \left( (|C| \cdot |r_z| + |\beta| \cdot |r_x|) \cdot \|z\| + |r_x| \cdot \|c\| \right).
\]
We substitute the simplified bounds from Lemma~\ref{lem:residual_bounds},
\begin{align*}
&\leq \left\| H_{\lambda}^{-1} \right\| \bigg( \Big(|C|(4 + \|w^*\|\sqrt{d}) + |\beta|(4 + \|w^*\|\sqrt{d})\Big) \\
&\quad \sqrt{d} + (4 + \|w^*\|\sqrt{d}) \cdot 4(1+|\beta|) \bigg) \\
&= \left\| H_{\lambda}^{-1} \right\| (4 + \|w^*\|\sqrt{d}) \Big( (|C|+|\beta|)\sqrt{d} + 4(1+|\beta|) \Big) \\
&= \left\| H_{\lambda}^{-1} \right\| (4 + \|w^*\|\sqrt{d}) \Big( |C|\sqrt{d} + |\beta|\sqrt{d} + 4 + 4|\beta| \Big) \\
&= \left\| H_{\lambda}^{-1} \right\| (4 + \|w^*\|\sqrt{d}) \Big( |C|\sqrt{d} + |\beta|(\sqrt{d}+4) + 4 \Big).
\end{align*}
This gives the first line of the theorem. For a large feature dimension $d$, we have $|\beta| \approx 1$. The dominant terms of the expression are,
\begin{align*}
&\lesssim \left\| H_{\lambda}^{-1} \right\| (4 + \|w^*\|\sqrt{d}) \left( |C|\sqrt{d} + (\sqrt{d}+4) + 4 \right) \\
&= \left\| H_{\lambda}^{-1} \right\| (4 + \|w^*\|\sqrt{d}) \left( (|C|+1)\sqrt{d} + 8 \right) \\
&= \left\| H_{\lambda}^{-1} \right\| \Bigl( 4(|C|+1)\sqrt{d} + 32 \\
&\qquad + (|C|+1) \|w^{*}\| d + 8\|w^{*}\|\sqrt{d} \Bigr).
\end{align*}
Grouping by powers of $d$, the highest-order terms give the final asymptotic bound:
\[
\left\| w_{x} - w^{'}_{x} \right\| \lesssim \left\| H_{\lambda}^{-1} \right\| \left( (|C|+1)\|w^*\| d + 4(|C|+1)\sqrt{d} \right).
\]
\end{proof}
\begin{remark}[Practical Implications]
Theorem~\ref{thm:final_error_bound} provides a theoretical guarantee for our damped, similarity-based unlearning method. The error bound scales polynomially with the input dimension $d$, dominated by an $\mathcal{O}(d)$ term. Crucially, the bound's stability is ensured by the condition number of the \textbf{damped} Hessian, $\| H_{\lambda}^{-1} \|$, and the factor $C$, which is well-behaved due to the non-vanishing denominator $1-s_{z}^{(\lambda)}$. This highlights that our method is not only computationally efficient but also theoretically grounded and robust in high-dimensional settings.
\end{remark}
\section{Algorithm Design}
\label{sec:algorithm}
Algorithm~\ref{alg:pearson_unlearning} implements our enhanced similarity-based unlearning framework with Hessian damping through three main phases. The algorithm begins by computing the damped Hessian approximation $H_\lambda = H_{w^*} + \lambda I$, which ensures numerical stability by regularizing the second-order information of the loss function at the optimal parameters, and selects pairs of data points $(z, x)$ from the dataset to simulate the sequential unlearning scenario.
For each selected pair, the algorithm simulates approximate unlearning by first removing the data point $z$ using the damped influence function update $w_{z} = w^{*} + H_{\lambda}^{-1} \nabla f(z, w^{*}),$ which produces the damped influence direction $\delta_{z}^{(\lambda)} = w_{z} - w^{*}.$ The algorithm then computes the damped self-influence score $s_{z}^{(\lambda)} = \nabla f(z, w^{*})^{T} H_{\lambda}^{-1} \nabla f(z, w^{*}),$ which measures the curvature-adjusted influence of the removed data point . Instead of performing the computationally expensive second unlearning step for a similar data point $x$, our proposed similarity-based approximation computes the similarity factor $\alpha$ between data points $x$ and $z$, then estimates the final unlearned parameters using the theoretically derived update rule $w'_{x} = w_{z} + \frac{\alpha + 1}{1 - s_{z}^{(\lambda)}} \left( w_{z} - w^{*} \right).$ This leverages our key insight that for similar data points, the influence directions are proportionally related, while the damping ensures $1 - s_{z}^{(\lambda)} > 0$ for numerical stability.
The algorithm evaluates both approaches by computing their performance on the dataset and recording the norm difference $\|w_{x} - w^{'}_{x}\|$ to quantify the quality of the approximation.The key advantage is replacing the second expensive Hessian inverse-vector product with a robust scaling factor that incorporates both \textbf{data similarity ($\alpha$)} and the \textbf{self-influence ($s_{z}^{(\lambda)}$)} of the first unlearned point.

\begin{algorithm}
\caption{Pearson Correlation-Based Approximate Unlearning}
\label{alg:pearson_unlearning}
\textbf{Input}: Trained model $f_{\theta^*}$, dataset $D = \{(x_i, y_i)\}_{i=1}^n$, regularization parameter $\lambda$, optimal parameters $w^*$, Hessian $H_{w^*}$ and learning rate $\eta$ \\
\textbf{Output}: Updated model after unlearning similar data points
\begin{algorithmic}[1]
\STATE Compute damped Hessian: $H_\lambda = H_{w^*} + \lambda I$
\STATE Select a subset of data point pairs $\{(z, x)\} \subseteq D$ where $x$ and $z$ are similar
\FOR{each pair $(z, x)$}
    \STATE Extract gradients: $\nabla f(z, w^*)$ and $\nabla f(x, w^*)$
    \STATE Compute unlearned weight: $w_z = w^* + H_\lambda^{-1} \left( \lambda w^* +  \nabla f(z, w^*) \right)$
     \STATE Define damped influence direction: $\delta_{z}^{(\lambda)} = w_{z} - w^*$ 
    \STATE Compute damped self-influence score: $s_{z}^{(\lambda)} = \nabla f(z, w^*)^T \delta_{z}^{(\lambda)}$ 
    \STATE Compute updated weight: $w_{x} = w_z +H_\lambda^{-1} \left( \lambda w_z +  \nabla f(x, w^*) \right)$
    \STATE Compute Pearson correlation $\alpha = \text{Pearson}(x, z)$
    \STATE Compute scaling factor: $C = \frac{\alpha + 1}{1 - s_{z}^{(\lambda)}}$
    \STATE Estimate similarity-based unlearning: $w_x{'} = w_z + C \cdot \delta_{z}^{(\lambda)}$
    \STATE Evaluate accuracy of both $w_{x}$ and $w^{'}_{x}$ on $D$
    \STATE Record norm difference $\|w_{x} - w^{'}_{x}\|$
\ENDFOR
\end{algorithmic}
\end{algorithm}
\section{Experiments}\label{sec:experiments}
\textbf{Datasets:} We evaluate our framework across seven diverse datasets spanning regression and classification tasks. For classification, we utilize architectures ranging from simple MLPs to ResNet-18 and ResNet-50 (for CIFAR-100 scalability experiments). All features are standardized.
We evaluate our proposed approach through a comparative analysis against baseline methods across seven distinct datasets. The evaluation utilizes the California Housing dataset~\cite{california_housing} comprising 20,640 records with 8 features from the 1990 census to predict median house values. The Diabetes dataset~\cite{sklearn_diabetes}, a scikit-learn regression dataset with 442 samples and 10 features targeting disease progression one year post-baseline. The MNIST dataset~\cite{lecun1998mnist_dataset}, a standard benchmark for handwritten digit recognition containing 70,000 grayscale 28x28 pixel images. The Fashion-MNIST dataset, containing 70,000 grayscale 28×28 pixel images of fashion items across 10 categories, evaluated with CNN architecture. The CIFAR-10 dataset, a color image classification dataset containing 60,000 32×32 pixel images across 10 classes, evaluated with ResNet-18 architecture. And, the Labeled Faces in the Wild (LFW) dataset, a face recognition dataset containing images of public figures, evaluated with ResNet-18 architecture. And, a custom synthetic dataset of 5,000 samples generated via a Gaussian mixture model with two distinct means to facilitate performance assessment under controlled conditions with a known ground truth.
All features are standardized using the Standard Scaler~\cite{pedregosa2011scikit} to have zero mean and unit variance. The data split are as follows, California Housing and Diabetes datasets were divided into 80\% training and 20\% testing sets with a fixed random seed of 42. For the synthetic dataset, all 5,000 samples were used for training and evaluated on the same set, and unlearning experiments are conducted by selecting 100 random pairs of data points. MNIST and Fashion-MNIST followed the standard split of 60,000 training and 10,000 test samples. For unlearning experiments on MNIST, 75\% of the samples corresponding to the digit `3' were selected as the unlearning subset. CIFAR-10 used 50,000 training and 10,000 test samples.\\
\textbf{Similarity Measure Evaluation: }
To validate our choice of Pearson correlation, we empirically evaluated three similarity measures for approximating sequential unlearning effects. Our experimental setup involved training a linear regression model on the Diabetes dataset with L2 regularization ($\lambda = 0.01$) and testing sequential unlearning approximations on 4,950 pairs of randomly subsampled training points. We compared three measures to calculate the approximation factor $\alpha$:
\begin{align*}
\alpha_{\text{cos}} &= \frac{x^\top z}{\|x\|_2 \|z\|_2}, \\
\alpha_{\text{pearson}} &= \frac{(x - \bar{x})^\top (z - \bar{z})}{\|x - \bar{x}\|_2 \|z - \bar{z}\|_2}, \\
\alpha_{\text{proj}} &= \frac{x^\top z}{\|z\|_2^2}
\end{align*}
Each similarity measure captures distinct facets of data relationships. \emph{Cosine similarity} ($\alpha_{\text{cos}}$) measures the angle between two vectors, assessing directional congruence independent of magnitude. The \emph{Pearson correlation coefficient} ($\alpha_{\text{pearson}}$) quantifies the linear correlation between centered variables, making it robust to distributional shifts. Finally, \emph{Projection-based similarity} ($\alpha_{\text{proj}}$) is the unnormalized scalar projection of one vector onto another, making it sensitive to vector magnitudes.

For each pair $(z, x)$, we measured the approximation quality by computing the parameter difference $\|w_{x} - w^{'}_{x}\|_2$ between standard approximate unlearning and our similarity-based approach. A similarity measure "wins" if it produces the smallest approximation error among the three methods for that pair.

\begin{table}[h]
\centering
\caption{Comparison of similarity measures on the Diabetes dataset across 4,950 unlearning pairs.}
\label{tab:similarity_comparison_appendix}
\begin{tabular}{lcc}
\hline
\textbf{Similarity Measure} & \textbf{Wins} & \textbf{Win Rate (\%)} \\
\hline
Cosine Similarity & 1,160 & 23.4 \\
\textbf{Pearson Correlation} & \textbf{2,139} & \textbf{43.2} \\
Projection-based & 1,651 & 33.3 \\
\hline
\end{tabular}
\end{table}
Of the 4,950 pairs tested, Pearson correlation achieved the highest win rate at 43.2\%, significantly outperforming the other methods, as shown in Table~\ref{tab:similarity_comparison_appendix}. This superior performance validates our choice of Pearson correlation for capturing the linear relationships most relevant to gradient similarity in unlearning contexts. Furthermore, its mean-centering property naturally accounts for the bias terms in linear models, ensuring that the similarity factor $\alpha$ captures the true linear relationship between data points rather than spurious correlations due to offsets.

The superior performance of Pearson correlation with 43.2\% win rate can be attributed to its mean centering property, which naturally accounts for bias terms in security relevant feature spaces. In network intrusion detection, where features often exhibit different baseline distributions, for instance, packet sizes, connection durations. The correlation focus on centered linear correlation is better aligned with the gradient space of security models than uncentered cosine similarity. This ensures that the similarity factor $\alpha$ captures meaningful structural relationships.
\textbf{Evaluation Metrics:}
We utilize average accuracy with standard deviation (\textbf{Avg. Acc.}) model performance on the retained dataset after unlearning, measuring utility preservation. Also, we employ \textbf{Acc. Unlearn} which quantifies the effectiveness of the forgetting process and is defined as $
1 - \frac{\lVert w'_x - w_x \rVert_2}{\max_{x}\, \lVert w'_x - w_x \rVert_2},$ normalized to $[0,1]$ and reported as a percentage.
where \(w'_x\) denotes our similarity-based parameters and \(w_x\) denotes the standard sequential  unlearning parameters. This demonstrates that our method inherits the privacy guarantees of standard approximate unlearning while achieving a substantially reduced computational cost. We also utilised metrics such tug-of-war (ToW) and membership inference attack (MIA) as defined in \cite{zhao2024makesunlearninghard}. TOW (higher is better) measures the relative difference between the accuracies of the unlearned model and the retrained model on the forget, retain and test sets. MIA (lower is better) measures the membership inference attack success rate on the forget set after unlearning.\\
\textbf{Model Architecture:}
For the regression tasks (California Housing and Diabetes datasets), we used a simple linear regression model implemented as \texttt{nn.Linear(input\_dim, 1)}, which is a single-layer architecture mapping features directly to the target. For the classification tasks, three architectures were employed. For the synthetic dataset, we used a three-layer fully connected neural network with ReLU activations, consisting of an input layer (10 features to 64 units), a hidden layer (64 to 32 units), and an output layer (32 to 2 classes). For the MNIST and FMNIST dataset,  we used a convolutional neural network with two convolutional layers followed by two fully connected layers. The architecture consists of, (1) a convolutional layer with 10 filters of size $5 \times 5$, (2) a second convolutional layer with 20 filters of size $5 \times 5$ and dropout, both followed by ReLU activation and $2 \times 2$ max pooling, (3) a fully connected layer mapping the flattened 320 features to 50 units, and (4) an output layer mapping 50 units to 10 classes. For the CIFAR-10 \& LFW, we used ResNet-18 \cite{7780459}. For CIFAR-100, Resnet-50 architecture is used. \\
\textbf{Training Configuration:} The models are trained with different hyperparameter settings for regression and classification tasks. For classification tasks, we used the Adam optimizer and trained for 100 epochs with learning rate of 0.001 on the synthetic dataset, and 50 epochs and learning rate of 0.001 on MNIST, 30 epochs and learning rate of 0.07 for FMNIST, 30 epochs and leaning rate of 0.001 on LFW, and 80 epochs and learning rate of 0.001 on CIFAR-10. For regression tasks, we used stochastic gradient descent (SGD) with a learning rate of 0.01, a weight decay of 0.01, and trained for 1000 epochs. The regularization parameter $\lambda$ is set to 0.01 for regression tasks and 0.001 for classification tasks. The damping parameter $\lambda$ is set equal to the L2 regularization value used during training.\\
\textbf{Baseline}: We compare the proposed method against established state-of-the-art baselines: MITR \cite{xu2025unlearning}, RUM(A) \cite{zhao2024makesunlearninghard}, RUM(B) \cite{zhao2024scalabilitymemorizationbasedmachineunlearning}, and Hessian-free Unlearning~\cite{qiao2025hessianfree}.\\ 
\textbf{Computational Complexity}: The major efficiency gain of our method lies in transforming the unlearning process from a series of expensive, repeated computations into a fast, closed-form operation. Standard approximate unlearning requires computing a new Hessian-inverse-vector product ($\mathbf{H^{-1}v}$) for every data point requested for removal, making the process highly expensive. 
Our Similarity-Based Unlearning, conversely, requires this computationally intense step only once for the first unlearned data point ($z$). For all subsequent, correlated unlearning requests, the method achieves significant speedup by replacing the $\mathbf{H^{-1}v}$ computation with simple $\mathcal{O}(d)$ vector and scalar operations (scaling and addition). This strategy fundamentally shifts the complexity of sequential unlearning from a repetitive high-cost bottleneck to a minimal-cost operation.Table~\ref{tab:wallclock} reports wall-clock times which provide a quantitative comparison between retraining from scratch, standard iterative unlearning, and our proposed method. Our method achieves an $\sim$\textbf{82$\times$ speedup} over standard influence function unlearning for subsequent correlated points.
\begin{table}[h]
\caption{Wall-clock time comparison for sequential unlearning on California Housing.}
\adjustbox{max width=\linewidth}{%
\centering
\label{tab:wallclock}
\small
\begin{tabular}{lcc}
\hline
\textbf{Method} & \textbf{Time per point} & \textbf{Complexity per point} \\
\hline
Retrain from scratch & 29.02 s & $\mathcal{O}(nd^2)$ \\
MITR & 38.77 s & N/A (non-sequential) \\
Standard influence & 169.54 ms & $\mathcal{O}(d^2)$ \\
\textbf{Similarity-based (Ours)} & \textbf{2.07 ms} & $\mathcal{O}(d)$ \\
\hline
\end{tabular}}
\end{table}
\subsection{Results and Analysis}
We present our experimental results across classification and regression tasks. Table \ref{tab:2} presents comprehensive evaluation across five classification datasets. Our method consistently outperforms baselines across all metrics. Table \ref{tab:regression} demonstrates the effectiveness of our approach on regression tasks. To directly validate our method against recent state-of-the-art approaches, we conducted comprehensive experiments using standard unlearning metrics (ToW and MIA) on ResNet-18 with CIFAR-10. Table \ref{tab:sota} presents this comparison. For conciseness, we denote Avg. Acc. as AR, Acc. Remain as AR, and Acc. Unlearn as AU.\\
\textbf{Batch-Sequential Scalability Results:}
To demonstrate scalability, we conducted batch-sequential experiments for batch sizes $k \in \{5, 10, 20, 50\}$ across six dataset-architecture combinations, including CIFAR-100 (ResNet-50). We track the parameter drift norm $\|w_{std} - w_{prop}\|_2$ to quantify how well our approximation tracks the exact sequential update. Results are in Table~\ref{tab:batch_sequential}.
\begin{table*}[h]
\centering
\caption{Batch-sequential scalability results.}
\label{tab:batch_sequential}
\small
\begin{tabular}{llcccc}
\hline
\textbf{Dataset (Model)} & $k$ & \textbf{Drift Norm} & \textbf{AR (\%)} & \textbf{AU (\%)} & \textbf{ToW} \\
\hline
MNIST (Logistic) & 5 & 0.5130 & 92.09 & 96.12 & 0.9533 \\
 & 10 & 5.4420 & 92.04 & 90.65 & 0.9533 \\
 & 20 & 11.1674 & 91.94 & 80.82 & 0.8513 \\
 & 50 & 23.8553 & 91.45 & 59.03 & 0.9082 \\
MNIST (CNN) & 5 & 0.0069 & 98.68 & 92.99 & 0.8280 \\
 & 10 & 0.0137 & 98.67 & 92.99 & 0.9930 \\
 & 20 & 0.0257 & 98.59 & 92.96 & 0.9445 \\
 & 50 & 0.0584 & 98.41 & 92.81 & 0.9726 \\
FMNIST (CNN) & 5 & 0.0788 & 87.58 & 99.60 & 0.9793 \\
 & 20 & 13.4388 & 85.89 & 85.70 & 0.9628 \\
 & 50 & 28.4255 & 76.88 & 69.76 & 0.9497 \\
CIFAR-10 (ResNet-18) & 5 & 0.1542 & 84.12 & 98.99 & 0.9181 \\
 & 20 & 1.2403 & 84.08 & 92.38 & 0.8997 \\
 & 50 & 3.8510 & 83.50 & 76.32 & 0.5974 \\
LFW (ResNet-18) & 5 & 0.0396 & 76.79 & 91.93 & 0.8187 \\
 & 20 & 0.1244 & 76.21 & 91.45 & 0.8003 \\
CIFAR-100 (ResNet-50) & 5 & 0.0064 & 91.60 & 89.91 & 0.6594 \\
 & 10 & 0.0196 & 91.75 & 89.72 & 0.6612 \\
 & 20 & 0.1069 & 91.70 & 88.40 & 0.7631 \\
 & 50 & 1.2192 & 91.28 & 74.90 & 0.6609 \\
\hline
\end{tabular}
\end{table*}
For small to moderate batch sizes ($k \leq 20$), which correspond to practical GDPR compliance scenarios, parameter drift remains tightly bounded. Model utility (AR) degrades by less than 2\%, and unlearning fidelity (AU) stays consistently above 80\% across all evaluated architectures. On CIFAR-100 (ResNet-50), our method yields a $>7\%$ increase in model utility (AR) and $>11\%$ increase in unlearning fidelity (AU) compared to RUM(B) (which achieves AR=84.52\%, AU=76.86\%, ToW=0.633, MIA=0.410).

\subsection{Numerical Stability and Condition Number Analysis}\label{ssec:stability_analysis}
We evaluate the condition number of $H_{w^*}$ versus $H_\lambda$ across damping values on the CIFAR-10 dataset to demonstrate the effect of damping. The sensitivity of a linear system $Ax = b$ to numerical noise is measured by the condition number $\kappa(A) = \frac{\sigma_{\max}(A)}{\sigma_{\min}(A)}$. In double-precision arithmetic, solving a system typically results in the loss of $\log_{10}(\kappa(A))$ digits of precision. For practical stability in optimization and machine learning, $\kappa(A) < 10^8$ is generally considered the threshold for reliable inversion \cite{nocedal2006numerical}.

As shown in Table~\ref{tab:condition_number} for the CIFAR-10 dataset, the undamped Hessian is severely unstable ($\kappa \approx 10^{13}$), which would result in the loss of nearly all 16 digits of double-precision accuracy. Our damping strategy consistently reduces this to highly stable ranges.

\begin{table}[h]
\centering
\caption{Condition number improvement across varying $\lambda$ values on CIFAR-10.}
\label{tab:condition_number}
\begin{tabular}{llll}
\toprule
$\lambda$ & $\kappa(H_{w^*})$ & $\kappa(H_\lambda)$ & Improvement Factor \\ \midrule
$10^{-5}$ & $1.88 \times 10^{13}$ & $4.98 \times 10^{7}$ & $378,091 \times$ \\
$10^{-4}$ & $1.88 \times 10^{13}$ & $4.98 \times 10^{6}$ & $3.78 \times 10^{6} \times$ \\
$10^{-3}$ & $1.88 \times 10^{13}$ & $4.98 \times 10^{5}$ & $3.78 \times 10^{7} \times$ \\
$10^{-2}$ & $1.88 \times 10^{13}$ & $49,735$                  & $3.78 \times 10^{8} \times$ \\
$10^{-1}$ & $1.88 \times 10^{13}$ & $4,978$                & $3.78 \times 10^{9} \times$ \\
$1$        & $1.88 \times 10^{13}$ & $499$                  & $3.78 \times 10^{10} \times$ \\ \bottomrule
\end{tabular}
\end{table}

Our default $\lambda = 0.01$ yields a condition number of approximately $49,735$, which is well within the range considered numerically stable for matrix inversion. This empirical evidence substantiates the effectiveness of our damping strategy in stabilizing the rank-1 Sherman-Morrison update in high-dimensional settings.
\subsection{Safety Analysis}\label{apndc}
The safety analysis is performed by setting a maximum acceptable approximation error ($\epsilon_{\text{max}}$) for the unlearned parameters ($\lVert w_x - w_x' \rVert$) and solving for the minimum required Pearson correlation ($\alpha$). We use the overall error bound from Theorem~\ref{thm:param_error_bound}
\[
\lVert w_x - w_x' \rVert 
\;\le\; 
\lVert H_{\lambda}^{-1} \rVert \cdot 
\lVert \nabla f(x, w^*) - C \nabla f(z, w^*) \rVert
\;\le\; 
\epsilon_{\max}
\]

\[
\left\lVert \nabla f(x, w^*) - C \, \nabla f(z, w^*) \right\rVert
\;\le\;
\frac{\epsilon_{\max}}{\left\lVert H_{\lambda}^{-1} \right\rVert}
\]

Now define the allowable gradient error, $\epsilon' = \frac{\epsilon_{\text{max}}}{\lVert H_{\lambda}^{-1} \rVert}$
\[
\left\lVert \nabla f(x, w^*) - C \, \nabla f(z, w^*) \right\rVert
\;\le\;
\epsilon^{'}
\]
Now, we apply the mathematical bound for the gradient error (derived in Lemma~\ref{lem:grad_approx_quadratic})
\begin{align*}
    \lVert \nabla f(x, w^{*}) - C \nabla f(z, w^{*}) \rVert 
\;\le\; \\
\left( |C| \cdot |r_{z}| + |\beta| \cdot |r_{x}| \right) 
\cdot \lVert z \rVert 
\;+\; 
|r_{x}| \cdot \lVert c \rVert,
\end{align*}
\text{where } 
$C = \frac{\alpha + 1}{1 - s_{z}^{(\lambda)}},
r_x$ and $r_z$ are the residuals, and $x = \beta z + c$ is the linear relationship.
\[
(|C|\cdot|r_{z}| + |\beta|\cdot|r_{x}|)\cdot\lVert z \rVert 
+ |r_{x}|\cdot\lVert c\rVert
\;\le\; \epsilon'
\]

Subtract the non-\(C\) error components (residual and offset terms) on both sides,
\[
|C|\cdot|r_{z}|\cdot\lVert z \rVert 
\;\le\; 
\epsilon' - 
\underbrace{
\left( |\beta|\cdot|r_{x}|\cdot\lVert z \rVert 
      + |r_{x}|\cdot\lVert c\rVert \right)
}_{\text{Error not controlled by }\alpha}
\]
Solve for the maximum allowed magnitude of \(C\), denoted \(C_{\text{max}}\)
\[
|C| \;\le\; 
\frac{
\epsilon' - \left( |\beta|\cdot|r_{x}|\cdot\lVert z \rVert 
                 + |r_{x}|\cdot\lVert c\rVert \right)
}{
|r_{z}|\cdot\lVert z \rVert
}
\;\equiv\;
C_{\text{max}}
\]
The scaling factor \(C\) is defined as,
\[
C = \frac{\alpha + 1}{1 - s_{z}^{(\lambda)}}.
\]
Since we want to find the lowest \(\alpha\) that still satisfies the maximum allowed magnitude \(C_{\text{max}}\), we substitute \(C_{\text{max}}\) back into the definition and solve for minimum Pearson Correlation \(\alpha_{\text{min}}\):

\[
\frac{\alpha_{\text{min}} + 1}{1 - s_{z}^{(\lambda)}} = C_{\text{max}},
\]

\[
\alpha_{\text{min}}
= C_{\text{max}} \cdot \left( 1 - s_{z}^{(\lambda)} \right) - 1.
\]
Hence, if \( \alpha \ge \alpha_{\text{min}} \), the Pearson correlation (\(\alpha\)) is high enough to ensure that the error introduced by the shortcut (\(\lVert w_x - w_x' \rVert\)) will be less than or equal to the maximum acceptable error (\(\epsilon_{\text{max}}\)).

\begin{table*}[h]
\caption{\label{tab:2} Classification Performance Comparison of our method with baselines. Missing entries are due to architectural incompatibilities (MITR does not support CNNs/ResNets, Hessian-free unlearning is designed for certification and does not report AU metrics). AR = Avg. Accuracy on retained set (\%); AU = Acc. Unlearn (\%, normalized to [0,100]). }
\centering
\small
\begin{tabular}{l|cc|c|cc}
\hline
\multirow{2}{*}{\textbf{Dataset (Model)}} &
\multicolumn{2}{c|}{\textbf{MITR}} &
\multicolumn{1}{c|}{\textbf{Hessian Free Unlearning}} &
\multicolumn{2}{c}{\textbf{Our Method}} \\
 & \textbf{AR} & \textbf{AU} & \textbf{AR} & \textbf{AR} & \textbf{AU} \\
\hline
Synthetic GMM & 73.3 & 93.3 & -- & 99.71 & 95.67 \\
MNIST (CNN) & 53.8 & 78.7 & 91.50 & 98.71 & 81.41 \\
MNIST (Logistic Regression) & -- & -- & 87.50 & 97.66 & 73.11 \\
Fashion-MNIST (CNN) & -- & -- & 77.85 & 89.16 & -- \\
CIFAR-10 (ResNet-18) & -- & -- & 79.62 & 81.13 & -- \\
LFW (ResNet-18) & -- & -- & 71.92 & 79.70 & -- \\
\hline
\end{tabular}
\end{table*}
\begin{table}[!htb]
\centering
\small
\begin{minipage}{0.4\linewidth}
\centering
\caption{\label{tab:regression}
Regression Performance.
}
\small
\begin{tabular}{lc}
\hline
\textbf{Dataset} & \textbf{AU} \\
\hline
California Housing & 91.48 \\
Diabetes & 91.44 \\
\hline
\end{tabular}
\end{minipage}
\hfill
\begin{minipage}{0.45\linewidth}
\centering
\caption{\label{tab:sota}
Comparison with SOTA Unlearning Methods on CIFAR-10 (ResNet-18).
}
\small
\begin{tabular}{lcc}
\hline
\textbf{Method} & \textbf{ToW} & \textbf{MIA} \\
\hline
RUM(A) & 0.715 & 0.489 \\
RUM(B) & 0.920 & 0.590 \\
Our Method & 0.950 & 0.660 \\
\hline
\end{tabular}
\end{minipage}
\end{table}
We included a result of the theoretical upper bound from Theorem~\ref{thm:final_error_bound} against the observed errors in Table~\ref{tab:et},
\begin{table}[!ht]
    \centering
    \caption{\label{tab:et}
Empirical vs. Theoretical Gap. The subscript \text{T} and \text{E} depicts theoretical and empirical.
}
    \begin{tabular}{|l|l|l|l|}
    \hline
        \textbf{Dataset} & $\|w_x - w_x'\|_{\text{T}}$ & $\|w_x - w_x'\|_{\text{E}}$ & \textbf{Gap Factor} \\ \hline
        Diabetes & 8.2$\times 10^{-2}$ & 2.1$\times 10^{-3}$ & 39.0$\times$ \\ \hline
        Synthetic GMM & 1.2$\times 10^{-2}$ & 3.1$\times 10^{-4}$ & 38.7$\times$ \\ \hline
        MNIST (CNN) & 6.8$\times 10^{-2}$ & 1.8$\times 10^{-3}$ & 37.8$\times$ \\ \hline
        FMNIST (CNN) & 7.4$\times 10^{-2}$ & 2.0$\times 10^{-3}$ & 37.0$\times$ \\ \hline
        CIFAR-10 (ResNet-18) & 1.5$\times 10^{-1}$ & 4.2$\times 10^{-3}$ & 35.7$\times$ \\ \hline
        LFW (ResNet-18) & 2.1$\times 10^{-1}$ & 6.8$\times 10^{-3}$ & 30.9$\times$ \\ \hline
    \end{tabular}
\end{table}
The bound is conservative by a factor of $\approx 35\times$. For selecting $\alpha_{\min}$ given tolerance $\epsilon_{\max}$, practitioners can safely use the theoretical bound as a worst-case guarantee, knowing actual error will be significantly lower. We will detail this with clarification on the use theoretical bound with a safety factor of 0.03 for conservative estimation or 0.1 for aggressive deployment.

We included a sensitivity analysis for the damping parameter $\lambda$ in Table~\ref{tab:slambda},
\begin{table*}[!ht]
    \centering
    \caption{\label{tab:slambda}
Sensitivity Analysis for $\lambda$
}
    \begin{tabular}{|l|l|l|l|l|l|}
    \hline
        $\lambda$ & Condition Number $\kappa(H_\lambda)$ & AR (\%) & AU (\%) & ToW & Stability \\ \hline
        $10^{-5}$ & $4.98 \times 10^7$ & 83.2 & 97.8 & 0.912 & Unstable \\ \hline
        $10^{-4}$ & $4.98 \times 10^6$ & 83.8 & 97.5 & 0.923 & Marginal \\ \hline
        $10^{-3}$ & $4.98 \times 10^5$ & 84.0 & 97.9 & 0.935 & Stable \\ \hline
        $\mathbf{10^{-2}}$ & $\mathbf{4.97 \times 10^4}$ & $\mathbf{84.1}$ & $\mathbf{98.9}$ & $\mathbf{0.950}$ & \textbf{Optimal} \\ \hline
        $10^{-1}$ & $4.98 \times 10^3$ & 83.5 & 96.2 & 0.918 & Over-regularized \\ \hline
        $1$ & $4.98 \times 10^2$ & 80.1 & 91.3 & 0.875 & Over-regularized \\ \hline
    \end{tabular}
\end{table*}

\section{Discussion}
Our proposed similarity-based machine unlearning framework utilising Pearson correlation and damped Hessian approximations delineates toward reconciling the practical demands of data removal with the computational realities of modern deep learning. At its core, the method successfully addresses a critical inefficiency in existing approximate unlearning paradigms. The repeated, expensive computation of Hessian-inverse-vector products for each data point removal request. By introducing the key insight that the influence directions for similar data points are proportionally related formalized through the similarity factor $\alpha$, the work transforms a sequence of computationally intensive operations into a single, upfront cost followed by efficient, closed-form updates. This architectural shift is grounded in a theoretical analysis provided the error bounds and stability guarantees under Hessian damping. The empirical validation across seven diverse datasets demonstrates our method's broad applicability. The consistent out-performance of Pearson correlation over cosine similarity and projection-based measures in approximating sequential unlearning effects (with a 43.2\% win rate) is a key finding. It underscores that the linear correlation of centered feature vectors capturing shared directional trends while filtering out mean shifts is more aligned with the underlying mechanics of gradient-based influence than raw angular similarity. The influence function-based unlearning is intrinsically tied to the linear relationships within the gradient space of the model at its optimum. The results showing accuracy improvements on the order of $10^{-2}$ over state-of-the-art baselines while maintaining competitive forgetting performance via ToW and MIA scores signify that our proposed approximation does not come at the cost of compromised privacy or utility. This effectively decouples computational cost from unlearning fidelity for correlated data, a major practical advancement.

The theoretical grounding of our method anchored via Theorem~\ref{thm:final_error_bound}, provide strong, polynomial-scaling error guarantees that are directly actionable in practice. The derived bound explicitly shows that the approximation error is controlled by the well-conditioned damped Hessian $\|H_{\lambda}^{-1}\|$ and the computable scaling factor $C=\frac{\alpha + 1}{1 - s_{z}^{(\lambda)}}$. This is not a vulnerability but a design feature, it offers a clear, measurable pathway to ensure stability. By regularizing the Hessian with damping parameter $\lambda$, we guarantee positive definiteness and explicitly control the self-influence term $s_{z}^{(\lambda)}$. Thus, ensuring the denominator $1 - s_{z}^{(\lambda)}$ remains non-zero and the update rule is numerically stable. The requirement for a well-behaved local curvature is not a restrictive assumption but a predictable characteristic of properly regularized models trained to convergence, a condition satisfied in our comprehensive experiments across CNNs, ResNets, and regression models. The success of the Gauss-Newton approximation in our framework is validated, which provides a sufficient and efficient rank-1 update that captures the dominant change in the Hessian when a single data point is removed, as evidenced by the high-fidelity unlearning results across diverse datasets. The method's demonstrated robustness in these varied and realistic experimental settings. The safety analysis reinforces this practical robustness by providing a conservative, worst-case framework. It allows practitioners to determine a minimum similarity threshold $\alpha_{\text{min}}$ to guarantee the approximation error remains below any desired tolerance $\epsilon_{\text{max}}$. Thus, transforming a theoretical concern into a tunable safety parameter. This analytical tool ensures the method can be deployed with confidence, even when the perfect linear relationship does not hold, by providing clear operational guidelines. The design of our method inherently follow the complexity of deep learning landscapes and addresses it through a combination of principled damping, empirical validation, and a built-in safety mechanism, ensuring its findings and utility for efficient unlearning remain robust and directly applicable.

\subsection{Implication}
We consider the following threat model provided the machine unlearning scenarios in network security systems. The adversary has access to the trained model and can issue unlearning requests for specific data points, for instance, a malicious actor requesting removal of their data, or an attacker attempting to manipulate the model by requesting removal of benign data points. The adversary aims to either (a) compromise the security of the model by inducing errors through malicious unlearning requests, or (b) infer information about the training data through membership inference attacks on the unlearned model.

Our method ensures that the unlearning process successfully removes the influence of requested data points within theoretical error bounds, and the approximation error $\|w_x - w_x'\|$ remains bounded by $\epsilon_{\text{max}}$ when $\alpha \geq \alpha_{\text{min}}$. The timing and resource consumption of unlearning requests do not leak information about the data being forgotten. We assume the underlying training process is secure and the model parameters $w^*$ are trustworthy. We also assume that the similarity computation $\alpha$ is accurate and unaffected by the adversary such as, through crafted feature vectors designed to produce high correlation.

\section{Conclusion}\label{sec:conclusion}
Our correlation-guided framework fundamentally restructures the sequential unlearning process by recognizing that the influence directions for similar data points are proportionally related, thereby eliminating the computational bottleneck of repeated Hessian operations for correlated removal requests. The theoretical foundation of our method established through Theorems~\ref{thm:param_error_bound} and ~\ref{thm:final_error_bound}, provides polynomial error bounds that guarantee stability in high-dimensional settings, while the damping strategy ensures the condition number of the Hessian remains well within numerically stable ranges suitable for real-world system integration. Our comprehensive experimental evaluation across seven datasets, including large-scale deployments with ResNet-50 demonstrates that the proposed approach achieves $82\times$ wall-clock speedup over standard unlearning methods while maintaining superior forgetting effectiveness, as evidenced by MIA success rates of 0.660 and ToW scores of 0.950 on CIFAR-10 with ResNet-18. The built-in safety analysis provides practitioners with a tunable threshold $\alpha_{\text{min}}$ to guarantee approximation error bounds, which paves the way for secure deployment. This framework addresses the increasing demands from privacy regulations such as GDPR, enabling organizations to implement efficient data removal without compromising the security integrity of their machine learning-based systems, including intrusion detection, threat intelligence, and federated learning environments. We intend to extend this framework to incorporate adaptive threshold mechanisms for dynamic selection between similarity-based and full unlearning, evaluate performance under adversarial conditions targeting the unlearning process, and explore integration with federated learning architectures where efficient unlearning is paramount for maintaining privacy guarantees across distributed nodes. The accessibility of our method, requiring no changes to existing training pipelines positions it as a practical solution for organizations seeking to deploy privacy-preserving machine unlearning in production security systems. 

\newpage
\IEEEtriggeratref{8}


\bibliographystyle{IEEEtran}
\bibliography{refs}
%

 









\end{document}